\documentclass{article}
     \PassOptionsToPackage{round}{natbib}
     \usepackage[final]{custom}

\usepackage{microtype}
\usepackage{graphicx}
\usepackage{subfigure}
\usepackage{booktabs} 

\usepackage{tikz}
\usepackage{nicefrac}

\usepackage{amsmath,amsfonts,bm}

\def\eqref#1{equation~\ref{#1}}

\def\1{\bm{1}}

\DeclareMathAlphabet{\mathsfit}{\encodingdefault}{\sfdefault}{m}{sl}
\SetMathAlphabet{\mathsfit}{bold}{\encodingdefault}{\sfdefault}{bx}{n}

\usepackage{hyperref}
\usepackage{url}

\newcommand{\dd}{-1}
\newcommand{\ts}{\mathsf{T}}

\usepackage{graphicx}
\usepackage{color}
\definecolor{mahogany}{rgb}{0.65, 0., 0.5}
\definecolor{Yoni_color}{rgb}{0.3, 0.3, 0.9}

\graphicspath{{./Figures/},{./Experiments/}}

\usepackage{amsthm}
\newtheorem{theorem}{Theorem}

\theoremstyle{definition}
\newtheorem{example}[theorem]{Example}
\newtheorem{definition}{Definition}

\usepackage{algorithm}
\usepackage{algorithmic}
\usepackage{verbatim}

\usepackage[utf8]{inputenc} 
\usepackage[T1]{fontenc}    
\usepackage{hyperref}       
\usepackage{url}            
\usepackage{booktabs}       
\usepackage{amsfonts}       
\usepackage{nicefrac}       
\usepackage{microtype}      

\AtEndDocument{\refstepcounter{theorem}\label{finalthm}}

\title{Wasserstein Diffusion Tikhonov Regularization}

\author{Alex Tong Lin$^1$, Yonatan Dukler$^1$, Wuchen Li$^1$, and Guido Mont{\'u}far$^{1,2,3}$\\
$^1$Department of Mathematics and $^2$Department of Statistics, UCLA, CA 90095;\\ $^3$Max Planck Institute for Mathematics in the Sciences, 04103 Leipzig, Germany}

\begin{document}

\maketitle

\begin{abstract}
We propose regularization strategies for learning discriminative models that are robust to in-class variations of the input data. We use the Wasserstein-2 geometry to capture semantically meaningful neighborhoods in the space of images, and define a corresponding input-dependent additive noise data augmentation model. Expanding and integrating the augmented loss yields an effective Tikhonov-type Wasserstein diffusion smoothness regularizer. This approach allows us to apply high levels of regularization and train functions that have low variability within classes but remain flexible across classes. We provide efficient methods for computing the regularizer at a negligible cost in comparison to training with adversarial data augmentation. Initial experiments demonstrate improvements in generalization performance under adversarial perturbations and also large in-class variations of the input data. 

\end{abstract}

\section{Introduction}

The sensitivity of trained discriminative models to small perturbations of the input data has become a reason of concern and an important topic of research in recent years~\citep{%
42503,
DBLP:conf/cvpr/NguyenYC15}. 
In particular, it has been observed that neural networks which have been trained to have good test performance can be fooled when the inputs are slightly perturbed in a way that is imperceptible to humans. 
This indicates a poor generalization ability, and specifically, that the solutions found with naive training and validation techniques are not capturing appropriate smoothness priors over the input space. 
A number of recent works have proposed approaches to improve robustness to perturbations~\citep{%
DBLP:conf/icml/CisseBGDU17,
DBLP:conf/cvpr/LiaoLDPH018,
samangouei2018defensegan,
DBLP:journals/corr/abs-1809-08516,
finlay2019improved}, 
while a complementary line of work probes the limitations of trained networks 
\citep{Salman2019ACR} and develops strategies to generate adversarial attacks~\citep{%
DBLP:journals/corr/abs-1709-10207,
DBLP:conf/cvpr/Moosavi-Dezfooli16,  
DBLP:conf/cvpr/Moosavi-Dezfooli17, 
shafahi2018are}. 

Intuitively, a smoother function at fixed training accuracy should be more robust to perturbations of the input, including adversarial attacks. Therefore, one strategy is to train the discriminative function with smoothness regularizers, such as noise added to the training examples (adversarial or random) or penalizing the norm of the gradient with respect to the inputs. We note, however, that the notion of a `small' perturbation will strongly depend on how we decide to measure distances in the space of inputs. The gradient and its norm depend on the geometric structure that is laid on input space. 

While it is convenient to use the $L^2$ metric (Euclidean), it is well understood that many data types of interest are not Euclidean. In particular, the $L^2$ metric does not measure distances between images in the way that we perceive them. Changes that humans consider small, might correspond to changes that the classifier considers to be large in this metric. Moreover, it is clear that a discriminative function on image data should be more stable in certain directions and more variable in other directions. This distinction is not well captured by isotropic smoothness regularizers. 

To construct more effective smoothness regularizers, two general approaches come to mind: 
1) Measure distances in a metric representation of the raw inputs, $d_\phi(x,y)^2 = \sum_j |\phi(x)_j - \phi(y)_j|^2$, where $\phi$ is some feature representation function that might be trained separately from or together with the discriminative task. Examples in this direction include preprocessing of the inputs by downsampling~\citep{guo2018countering}, autoencoders, and approaches that regularize intermediate representations within the neural network that is being trained for the discriminative task, such as injecting noise in the layers of a ResNet \citep{DBLP:journals/corr/abs-1811-10745}. 
2) Measure distances directly on the inputs (or following light preprocessing), but use a metric that is reflective of our perception of the data. Both approaches allow for data driven specifications and also direct incorporation of prior knowledge about the domain. We focus on the input space approach with the Wasserstein distance, in particular the Wasserstein-2 metric and geometry. 

The Wasserstein distance is known to be an effective metric in the space of images, as demonstrated in image retrieval problems~\citep{Rubner:2000:EMD:365875.365881,solomon2015convolutional,solomon2014earth} and related applications  \citep{peyre2019computational}. 
In particular, the Wassersetin distance is robust to natural variations such as translations and independent noise added to the pixel values. Importantly, the Wasserstein distance exhibits a Riemannian metric structure. This allows us to define Wasserstein gradient penalties and effective Wasserstein Gaussian noise\footnote{Wasserstein Gaussians appear in the small time behavior of a process called Wasserstein diffusion, investigated in continuous \citep{vonrenesse2009} and discrete \citep{Li2018_geometry} states.} in the space of images. The Wasserstein metric depends on the specific location at which it is being evaluated, and can define neighborhoods with a reasonable degree of semantic meaning. See, for example, the Wasserstein geodesics balls illustrated in Figure~\ref{fig:spheres}. 

Our approach in this article is inspired by recent work \citep{82084} which introduces a Wasserstein ground metric for Wasserstein GANs and demonstrates that this facilitates training of  discriminators that are more stable to natural variations of image data. 
We suggest that regularization based on Wasserstein geometry can make a discriminative function noticeably smoother along the directions of natural variations of images, but without making it constant along the directions of semantic variation. 
Moreover, a generative perturbation model can be folded into the training objective (by computing the expectation value over perturbations of the Taylor expanded loss around each training example). This yields an effective penalty term that integrates (up to a given order in the expansion) a continuum of perturbations at once. The Wasserstein geometry on image space is described by a metric tensor whose inverse is a linear weighted Laplacian matrix. This fact allows us to compute the Wasserstein diffusion smoothness regularizer at a negligible cost. 

In Section~\ref{sec:adversarial} we discuss adversarial attacks. In Section~\ref{sec:noise} we discuss training with input noise. We propose a Wasserstein Gaussian distribution in image space, which reflects the natural local variability of images. Then we compute the expectation of the perturbed objective by Taylor expansions in the Wasserstein space. This leads to a Tikhonov-type Wasserstein diffusion smoothness regularizer. In Section~\ref{sec:experiments} we present preliminary experimental results. In Section~\ref{sec:related} we discuss relations of the proposed methods to some of the existing literature, and in Section~\ref{sec:discussion} we offer a discussion. 

\begin{figure}
\centering
\begin{tabular}{cc}
\begin{minipage}{.59\textwidth}
	\centering	
\begin{tabular}{cc}	
\includegraphics[width=.5\textwidth,clip=true,trim=0cm 0cm 6.3cm 0cm]{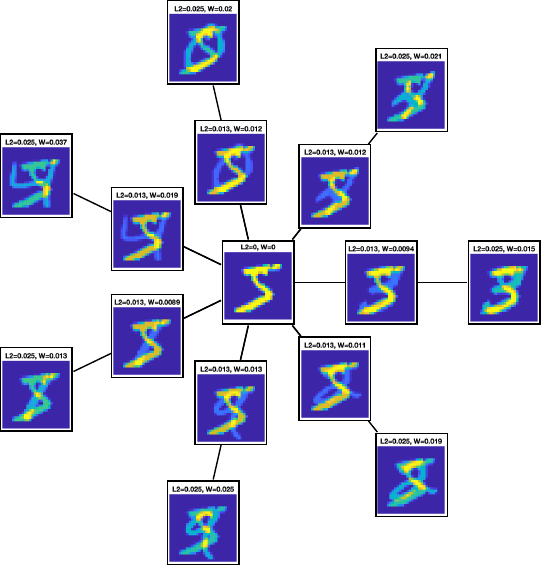} 
&
\!\!
\includegraphics[width=.5\textwidth,clip=true,trim=6.3cm 0cm 0cm 0cm]{spheres2}
\\
{\sffamily\footnotesize Euclidean} 
& 
\!\!
{\sffamily\footnotesize Wasserstein}
\end{tabular}
\end{minipage}
\;\;\;
& 
\begin{minipage}{.4\textwidth}
\sffamily
\begin{tikzpicture}
\draw[rotate=-5, fill=gray!10, draw =none, opacity =.7] (0,0) ellipse (1cm and 1cm); 
\draw[rotate=-5, fill=gray!30, draw = none, opacity =.7] (0,0) ellipse (.6cm and .6cm); 
\draw[rotate=-5, fill=blue!20, draw = none, opacity=.7] (0,0) ellipse (.45cm and 2.2cm); 
\draw[rotate=-5, thin] (0,0) ellipse (1cm and 1cm); 
\draw[rotate=-5, thin] (0,0) ellipse (.6cm and .6cm); 
\draw[rotate=-5, thin] (0,0) ellipse (.45cm and 2.2cm); 

\draw [very thick, dashed, xshift=.6cm, yshift=-2.4cm] (0,0) to [bend left=8] (.5,4.5) node [midway, right, sloped, yshift=.4cm] {\small \begin{minipage}{2.3cm}Ground\\ truth\\ boundary\end{minipage}};

\node at (0,0)[circle,fill,inner sep=1pt]{}; 
\node [above] at (0,0) {$x$};
\node at (.05,.8)[circle,fill,inner sep=1pt]{}; 
\node [above, xshift=.06cm, yshift=.05cm] at (.05,.85) {$\pi(x)$};

\node [left,fill=white,inner sep=1pt] at (-.55,-.5) {\small Euclidean ball};
\node [left,fill=white,inner sep=1pt] at (-.13,-1.5) {\small Wasserstein ball};
\end{tikzpicture}
\end{minipage}
\end{tabular}
	\caption{
		Shown are fixed-distance interpolates, measured by $L^2$ and Wasserstein metrics, between a source image and some other images in the MNIST data set. When training with smoothness regularizers towards achieving robustness against adversarial perturbations that are perceptually small, we need a good specification of semantically meaningful neighborhoods around which the discriminative function should not change much. As we see, the Wasserstein metric defines a more coherent neighborhood, with a natural interpretation in terms of continuous transportation of colors between pixels, which depends on the input example. 
		The right part illustrates the intuition that an appropriate metric can capture larger in-class neighborhoods around an input image $x$, and hence allow us to apply stronger smoothness regularization against perturbations $\pi(x)$ without hurting test performance. 
	}
	\label{fig:spheres}
\end{figure}

\section{Adversarial training and ground truth geometry}
\label{sec:adversarial}

An adversarial attack is a perturbed version $\pi(x)$ of an input example $x$, which alters the prediction of the classifier such that $f(\pi(x)) \neq f(x)$. 
According to this simplistic definition, every classifier can be successfully attacked, provided it has at least two possible output values. 
Taking a more refined perspective, consider $g(x)$ as the best possible classification (ground truth / Bayes classifier). 
Then a successful adversarial attack can be defined as a perturbation $\pi(x)$ of an input $x$ such that $f(x)=g(x)$ but $f(\pi(x))\neq g(\pi(x))$.  
This highlights that what we care about is not whether a classifier changes its prediction when the input is slightly perturbed, but rather in what scenario it changes its prediction. 

In order to quantify the sensitivity to attacks, we need a measure of the size of the perturbation model and the effect that it has on the classifier. 
Consider a loss function of the form 
\begin{equation}
E(f) = \mathbb{E}_{p(y|x)p(x)}\left[ l(f(x), y) \right]. 
\label{eq:loss}
\end{equation}
We can measure the detriment of the loss when the data is perturbed in comparison with the unperturbed loss. 
For generality, we define a perturbation $\pi$ at $x$ as a random variable. 
For example, in the case of samples from a vector space, additive perturbations take the form $\pi = x+\xi$, where $\xi$ can be, e.g., a zero mean Gaussian variable, or a deterministic value obtained by an attack strategy on the input $x$. 
The level of perturbation can be quantified, for example, in terms of the maximum or average size of the perturbations with respect to some norm, or in terms of the number of computations used to construct the attacks. 
The loss under perturbations is then 
\begin{equation}
\mathbb{E}_{p(\pi|x)p(y|\pi(x))p(x)} \left[ l(f(\pi(x)), y)\right]. 
\end{equation}
Unless we have access to the true labels given the perturbed inputs, i.e., the conditional probability $p(y|\pi(x))$ or the ground truth class $g(\pi(x))$, 
this measure is just theoretical. 
In practice the relation of labels to inputs is unknown, and for training we only have access to a training set $\{(x_i,y_i)\}_{i=1}^N$ and possibly the perturbation model $\pi$. 

A natural approach to obtaining more robust classifiers is to train with perturbations. 
A special case is adversarial training, where the perturbations are constructed specifically to deceive the classifier. 
The problem with this approach is that typically the correct class labels for the perturbed input data are unknown, as mentioned above. 
Therefore, typically one considers simply 
\begin{equation}
E_\pi(f) = \mathbb{E}_{p(\pi|x)p(y|x)p(x)}\left[ l(f(\pi(x)),y) \right].  
\label{eq:lossp}
\end{equation}
For a given pair $(x,y)$ with $y=g(x)$, in general $y \neq g(\pi(x))$. 
Therefore, the perturbation needs to be restricted in a way that ensures $g(\pi(x)) = g(x)$. 
The obvious and most common way to do this without using further information or prior knowledge about the ground truth, is to restrict the perturbations to be very small in some standard norm. 

Adversarial examples are often constructed by minimizing the confidence of the discriminative function or increasing the training loss with respect to the input. 
Since this does not incorporate prior knowledge about the shape of the ground truth, usually the perturbations are restricted to lie within a very small $L^p$ ball around the input example. 
Naive input noise regularization and gradient penalties suffer from a similar problem. When applied at the level that would be needed to prevent adversarial attacks, they tend to smear out the classifier in all directions around the input examples, leading to a significant detriment in test accuracy. 
A similar problem also arises in unsupervised adversarial training, where the base point is arbitrary (not necessarily a training example) and one requires that perturbations within a neighborhood are classified consistently. 
The situation is illustrated in the right part of Figure~\ref{fig:spheres}. 
Another difficulty is that data augmentation can be very expensive both in terms of the time it takes to compute each adversarial example and the number of examples that need to be added to the training data in order to obtain a sufficient level of robustness. 

Instead of restricting the perturbations to be small in an $L^p$ norm sense, we suggest to refine the metric on input space and the perturbation model. 
We propose to measure distances on input space using the Wasserstein metric and train with a corresponding Wasserstein Gaussian input noise. The Wasserstein metric assigns a small distance to natural local variations of an input image. This means that larger perturbations are more likely to remain within the class of the input example that is being perturbed. 
In turn, we can apply higher levels of noise, allow for larger size perturbations in adversarial training, or apply stronger gradient penalties during training. 
This is illustrated in Figure~\ref{fig:spheres}. 
In the next section we derive an effective regularizer for training with the Wasserstein metric on input space and which integrates the entire set of Wasserstein Gaussian noise perturbations (in a second order expansion) for each input example at once.

\section{Perturbed loss and Wasserstein diffusion Tikhonov regularizer} 
\label{sec:noise}

It is well known that training with input noise can be related to training the original objective with an added penalty~\citep{doi:10.1162/neco.1995.7.1.108}. 
These derivations usually are based on Taylor expansion of the perturbed loss around a given example. 
By default, the inputs are considered to live in Euclidean space, with loss functions such as the mean square error or the log loss. 
Following the arguments from the previous sections, we model the input space of images as a Wasserstein space. 
We then derive the Wasserstein Taylor expansion of the perturbed loss and the corresponding regularization penalties. Once the input space is regarded as a Wasserstein metric space, our derivations follow Riemannian calculus therein. 

We consider a perturbation model defined in terms of a ``Wasserstein Gaussian'', which at a given input image $x\in\mathcal{X}=\mathcal{P}(\Omega)$ has a density function of the form 
\begin{equation*}
p(\xi|x) = 
\exp(- d_W(x,x+\xi)^2 /\eta^2) d(\xi), 
\end{equation*}
with a scale parameter $\eta>0$ and a given reference measure $d(\xi)$. 
Here the Wasserstein-$q$ distance on image space can be defined in the linear programming formulation as 
\begin{equation}\label{ground_metric}
d_W(x,y):= \inf_{\Pi}  \left(\mathbb{E}_{(a,b)\sim \Pi}\left[ d_\Omega(a,b)^q \right]\right)^{\frac{1}{q}}, 
\end{equation}
where $\Pi$ is a joint distribution of pixel pairs $(a,b)$ with marginals $x, y$. Here
$x$, $y$ are images viewed as histograms over the set of pixels $\Omega$. 
The pixel ground metric $d_\Omega\colon \Omega\times \Omega \rightarrow \mathbb{R}_+$ assigns distances to pairs of pixels. The ground metric $d_\Omega$ can be defined in various ways that allow for efficient computations, and it can be trained from examples. 
We focus on the case $q=2$. In this case, we can extract from the Wasserstein distance $d_W$ a Riemannian metric for the space of images. Locally, the Wasserstein-$2$ distance can be expressed as 
\begin{equation}
d_W(x,x+\xi)^2 =(\xi, G_W(x) \xi)_{L^2} + o(\|\xi\|^2), 
\end{equation}
where $\int \xi(a)da=0$ and $G_W(x)=-\big(\nabla_a\cdot(x\nabla_a)\big)^{-1}$ is the Wasserstein Riemannian metric tensor at $x$. Here $\nabla$ and $\nabla\cdot$ are gradient and divergence operators in pixel space, with respect to $a\in\Omega$. In more details, 
\begin{equation*}
(\xi, G_W(x) \xi)_{L^2}=\int \|\nabla_a\Phi(a)\|^2x(a) da,    
\end{equation*}
where $\xi(a)=-\nabla\cdot(x(a)\nabla\Phi(a))$. Here the classical Wasserstein metric is defined on the space of images with equal total mass. This corresponds to the requirement $\int \xi(a) da=0$. In general, we can apply the metric in unnormalized image spaces as well; see related studies on unnormalized optimal transport \citep{GangboLiOsherPuthawala2019_unnormalizeda, LLLO}. 

In practice we will consider a discrete pixel space $\Omega=\{1,\ldots, n\}$. In this case, the Wasserstein Riemannian metric $G_W(x)$ is an $n\times n$ symmetric positive definite matrix that we will discuss further in the next section. 
Practically, the Wasserstein noise model corresponds to Gaussian noise with a covariance matrix $\eta^2 G_W^{-1}(x)$ that depends on the input $x$ and the choice of a ground metric $d_\Omega$ over pixels. 

We are now ready to present our theorem that relates training with Wasserstein diffusion to training with an added penalty term. 

\begin{theorem}[Perturbed loss regularization]
	\label{thm}
Consider an input space $(\mathcal{X},g)$ with the Riemannian metric $g$ represented by a matrix $G(x)$ depending on $x\in\mathcal{X}$, 
and consider the loss $E(f)=\mathbb{E}[l(f(x),y)]$ from~\eqref{eq:loss} with some error function $l$ that is twice differentiable in the first argument. 
Let $\xi$ be a Gaussian noise variable with zero mean and covariance matrix $\eta^2G^{-1}(x)$ depending on $x$. 
Then the perturbed loss from~\eqref{eq:lossp} takes the form 
	\begin{equation*}
	E_\xi(f) = E(f) + \frac{1}{2}\eta^2 E^R(f)  + o(\eta^2), 
	\end{equation*}
	where 
	\begin{equation*}
	E^R(f) = \mathbb{E}_{p(y|x)p(x)}\Big[ l''(f(x),y) \|\nabla_g f(x)\|_g^2 + l'(f(x),y) \Delta_g f(x) 
	\Big]. 
	\end{equation*}
Here $l'$ and $l''$ denote the first and second order ordinary partial derivatives of $l$ with respect to the first argument, and $\nabla_g$, $\|\cdot \|_g$, $\Delta_g$ are the gradient, norm, and Laplace-Beltrami operators on $(\mathcal{X},g)$. 
\end{theorem}

The proof is provided in Appendix~\ref{secA}. 
In this paper, we focus on $G(x)=G_W(x)$. 
This tells us that training with Wasserstein Gaussian noise corresponds, to second order in the noise level, to training with the unperturbed loss plus a penalty to the squared Wasserstein gradient norm and Laplace operators of the discriminative function. 

We are also interested in non-zero mean perturbations, such as adversarial perturbations used in adversarial training. 
In this case the proof of the theorem yields the expansion 
\begin{equation*}
E_\xi(f) = E(f) + E^R(f) + O(\|\xi\|^2), 
\end{equation*}
where 
\begin{equation*}
E^R(f) = \mathbb{E}_{p(y|x)p(x)} \left[ l'(f(x),y) \cdot \left( \mathbb{E}_{p(\xi|x)}[\xi]^\top \nabla f(x)\right) \right] . 
\end{equation*}	
If the perturbation is deterministic, $\mathbb{E}_{p(\xi|x)}[\xi]$ is simply $\xi(x)$. 
This suggests to regularize by penalizing the directional derivative of the discriminative function in the direction of the perturbation. 
If the perturbation is proportional to the Euclidean steepest descent direction of the discriminative function, $\xi\propto \nabla f(x)$, then the penalty is proportional to $\|\nabla f\|^2$. If the perturbation is proportional to the Riemannian steepest descent direction, $\xi\propto \nabla_g f(x) = G^{-1}(x)\nabla f(x)$, then the penalty is proportional to the Riemannian gradient norm squared, $\|\nabla_g f(x) \|_g^2$. 

\begin{example}[Square error]
For the square error $l(f(x),y) = (f(x)-y)^2$ and a perturbation model as in Theorem~\ref{thm}, we obtain the regularizer 
\begin{equation*}
	E^R(f) = \mathbb{E}_{p(y|x)p(x)}\Big[ \|\nabla_g f(x)\|_g^2 + (f(x) - y) \Delta_g f(x) 
	\Big]. 
\end{equation*}
For non-zero mean perturbations, we can consider an expansion to first order which gives 
\begin{equation*}
E^R(f) = \mathbb{E}_{p(y|x)p(x)} \left[ 2 (f(x) - y) \cdot \mathbb{E}_{p(\xi|x)}[\xi]^\top \nabla f(x) \right].  
\end{equation*}
\end{example}

\begin{example}[Cross entropy error]
For the cross entropy $l(f(x),y) = -y\ln (f(x)) - (1-y)\ln(1-f(x))$ and a perturbation model as in Theorem~\ref{thm}, we 
obtain the regularizer
\begin{equation*}
	E^R(f) = \mathbb{E}_{p(y|x)p(x)}\Big[ \left(\frac{y}{f^2(x)} + \frac{1-y}{(1-f(x))^2}\right) \|\nabla_g f(x)\|_g^2 +  \left(-\frac{y}{f(x)} + \frac{(1-y)}{(1-f(x))}\right) \Delta_g f(x) 
	\Big]. 
\end{equation*}
In the case of $k$ outputs (e.g., $k$-class classification), the loss function is simply the sum of the loss for each output times $1/k$. 
For non-zero mean perturbations we obtain 
\begin{equation*}
E^R(f) = \mathbb{E}_{p(y|x)p(x)} \left[ \left(-\frac{y}{f(x)} + \frac{1-y}{(1-f(x))}\right) \cdot \mathbb{E}_{p(\xi|x)}[\xi]^\top \nabla f(x) \right].   
\end{equation*}
\end{example}

\begin{example}[Euclidean inputs] 
In the case of Euclidean inputs and uncorrelated zero mean Gaussian noise of variance $\eta^2$, we recover some of the classic calculations by~\citet{doi:10.1162/neco.1995.7.1.108}. 
Consider as an example the square error function, for which the regularizer becomes 
\begin{equation*}
E^R(f) =  \mathbb{E}_{p(y|x)p(x)} \left[ \sum_i \left\{ \left(\frac{\partial f}{\partial x_i}\right)^2 +  (f(x) - y)\frac{\partial^2 f}{\partial x_i^2}  \right\}\right]. 
\end{equation*}
As pointed out by \citet{doi:10.1162/neco.1995.7.1.108}, this is the Tikhonov regularizer that is usually added to the sum of squares error. 
\end{example}

Theorem~\ref{thm} shows that all noise perturbed versions of a given input example can be integrated (in a second order sense) into a single term. 
Formally, equivalence of the regularizer to training with noise is only valid for small values of $\eta$, since it is based on a second order Taylor expansion. 
The Wasserstein diffusion smoothness regularizer $E^R$ also has the natural interpretation as decreasing the variability of the classifier in an anisotropic and input dependent way that is captured by the Wasserstein gradient norm and the Laplace-Beltrami operator. This interpretation remains valid for arbitrarily large values of $\eta$, even if in this case the regularized objective might no longer correspond to the integrated perturbed objective. 

We note that the term involving the Laplace-Beltrami operator is premultiplied with the derivative of the error. For regular choices of $l$, if the classifier makes correct predictions on the training inputs $x$ (which is often the case), the derivative $l'(f(x),y)$ will vanish. This suggests that for the purpose of regularization in settings where the training error vanishes, we can omit the Laplace-Beltrami term and consider only the gradient penalty. 

Taking the perspective of smoothness suggests that we may also regularize by penalizing the gradient of the discriminator, instead of the gradient of the loss function. 
Finally, we point out that the Wasserstein metric can also be used to define the size constraints for adversarial training. 
Usually adversarial perturbations are constrained to have $L^\infty$ norm (or some $L^p$ norm) bounded by a small $\epsilon$. 
Instead of using $\|\xi\|_{L^p}\leq \epsilon$, we can use $\|\xi\|_W \leq \epsilon$, or simply $\xi^\top G_W(x) \xi\leq \epsilon$.

\section{Training with the Wasserstein diffusion Tikhonov regularizer}

In this section we describe the implementation of the regularization scheme introduced in Section~\ref{sec:noise}. Theorem~\ref{thm} suggests to replace the original objective function by a regularized objective $E + \eta^2 E^R$, where $\eta$ corresponds to the strength of the regularization. For each training example $x$ we add 
\begin{equation}
l''(f(x),y)\cdot \|\nabla_W f(x)\|_W^2 + l'(f(x),y)\cdot \Delta_W f(x). \label{eq:regularization}
\end{equation}
The first term of \eqref{eq:regularization} involves the Wasserstein gradient of the discriminative function $f$ with respect to the input. This type of calculation appeared recently in the context of Wasserstein Generative Adversarial Networks with Wasserstein ground metric~\citep{82084}. 
We regard each input image as a histogram over pixels. 
Then we define a weighted graph $\mathcal{G}=(V, E, \omega)$, where the vertices $V=\{1,\ldots, n\}$ correspond to pixels, edges $E$ connect adjacent pixels, and $\omega$ is a symmetric matrix of weights $\omega_{ij}$ associated to the edges. The normalized volume form on node $i\in I$ is given by $d_i=\frac{\sum_{j\in N(i)}\omega_{ij}}{\sum_{i=1}^n\sum_{i'\in N(i)}\omega_{ii'}}$. For computational efficiency, we consider sparse graphs with a local connectivity structure that is invariant with respect to vertical and horizontal shifts in the pixel domain. In the experiments, we use local grids of radius $\operatorname{rad}=2,4,6,8$ and constant weights. 

The Laplacian matrix associated with the weighted graph $\mathcal{G}$ is defined, depending on the input $x$, as 
\begin{equation*}
L(x)_{ij}=\begin{cases}
\frac{1}{2}\sum_{k\in N(i)}\omega_{ik}(\frac{x_i}{d_i}+\frac{x_k}{d_k} ),& \textrm{if $i=j$}\\
-\frac{1}{2}\omega_{ij}(\frac{x_i}{d_i}+\frac{x_j}{d_j}),&\textrm{if $j\in N(i)$}\\
  0, & \textrm{otherwise}.
\end{cases}
\end{equation*}
The Wasserstein metric tensor is the matrix function given by the (pseudo) inverse of the weighted Laplacian operator, 
\begin{equation*}
G_W(x)=L(x)^{-1}\in\mathbb{R}^{n\times n}. 
\end{equation*}
Written explicitly, the Wasserstein gradient norm squared is 
\begin{align}
\|\nabla_W f(x)\|_W^2 
&= \nabla f(x)^\top G_W(x)^{-1}\nabla f(x) \nonumber\\
&=\nabla f(x)^\top L(x) \nabla f(x) \nonumber\\
&= \sum_{(i,j)\in E} \omega_{ij}\left(\frac{\partial}{\partial x_i} f(x) - \frac{\partial}{\partial x_j} f(x)\right)^2 \frac{x_i/d_i+x_j/d_j}{2}. \label{eq:grad_norm}
\end{align}
An efficient implementation of \eqref{eq:grad_norm} is described in Appendix~\ref{sec:computation_gradient}. 

The second term of \eqref{eq:regularization}, the Wasserstein Laplace-Beltrami of the discriminative function,  
is computed as 
\begin{equation*}
\Delta_W f(x) 
=\textrm{tr}(L(x)\nabla^2f(x))+\nabla f(x)^{\ts}L(x)\nabla\log\textrm{det}(L(x))^{-\frac{1}{2}}).
\end{equation*}
Here $\textrm{det}(L(x))^{-\frac{1}{2}}$ is the volume form for the Wasserstein manifold. Here $\textrm{det}(L(x))$ is the product of non-singular eigenvalues. 
For well-posedness of the volume form, we consider a compact set in the interior of a finite dimensional probability simplex. In this case, $L(x)$ is a positive definite matrix, and $\textrm{det}(L(x))^{-\frac{1}{2}}$ is well defined. In practice, we could also consider the reference density as the Lebesgue measure, which omits the volume term. Details on the implementation are provided in Appendix~\ref{app:Riemannian_calculus}. 
We leave a more systematic study and computation of the Wasserstein Laplace-Beltrami term for future work.

\section{Experiments}
\label{sec:experiments}

In this section we present initial experimental results, leaving a more extensive experimental evaluation for future work. 
We evaluate the utility of Wassserstein smoothness regularization in terms of the robustness of the trained classifiers to small and large perturbations. 
We focus on regularization by the Wasserstein gradient norm penalty.

\subsection{Stability to adversarial perturbations of the input data}

In this experiment we test the effectiveness of the gradient penalty regularizer in terms of the test accuracy of the trained classifiers. 
We train a ResNet-20 on clean images from CIFAR-10 with gradient norm penalty computed under Euclidean and Wasserstein metrics. 
We run grid search for the regularization strength and the radius defining the ground metric on pixel space. 
The training error converges to zero in all cases. 
We consider two types of test data: the clean test data set (natural generalization) and the test data set with each test example perturbed by an adversarial attack (robust generalization). 
Following current literature, adversarial perturbations are computed by FGSM and I-FGSM~\citep{DBLP:journals/corr/GoodfellowSS14,DBLP:conf/iclr/KurakinGB17}. 
More details on the implementation and hyperparameters are provided in Appendix~\ref{sec:details_exp1}. 
Our results, reported in Table~\ref{tab:my_label}, compare well with the state of the art for CIFAR-10, where current works report robust test 
error for FGSM $\epsilon = 8/255$ of $37.52 \%$, and for I-FGSM-20 $\alpha=2/255$, $\epsilon = 8/255$ of $42.06\%$ \citep{DBLP:journals/corr/abs-1811-10745}. 

\begin{table}[H]
    \centering
    \begin{tabular}{|l | c c c|}
    \hline
        Test data \quad\textbackslash{}\quad Regularizer & None & Euclidean grad. & Wassserstein grad. \\ \hline 
        Natural & $16.29$ & $15.61$ & $\mathbf{15.35}$ \\ \hline
        FGSM $\epsilon = 8/255$ &  $82.22$  & $31.10$ & $\mathbf{30.20}$ \\ \hline
        FGSM $\epsilon = 25/255$ &  $89.72$ & $66.83$ & $\mathbf{44.32}$ \\ \hline
        I-FGSM-20 $\alpha = 2/255$, $\epsilon=8/255$ &  
        $90.15$ & $40.06$ & $\mathbf{32.12}$ \\ \hline 
    \end{tabular}
    \smallskip
    \caption{Robust test error percentage (lower is better) for a ResNet-20 network with softplus  activation trained for 200 epochs on clean CIFAR-10 training images using gradient norm regularization with Euclidean and Wasserstein metric on image space. 
    We run grid search over the regularization strength and the ground metric radius on pixel space. 
    }
    \label{tab:my_label}
\end{table}

\subsection{Stability to large in-class variations of the input data}

Most work on adversarial robustness focuses on small perturbations, with adversarial attacks restricted to have a small norm so that they remain imperceptible to humans. 
We are interested in generalization for all kinds of in-class variations of the data, including large perturbations that should not change the predicted class. 
In this experiment we train on the clean CIFAR-10 training set (no data augmentation), and compare between no regularization, Euclidean, and Wasserstein smoothness regularization. 
For testing, we randomly draw $1000$ images from the natural test set and construct for each of them a sequence of translated versions with padding, as depicted in Figure~\ref{fig:horse}. 
The semantic meaning of images should remain constant under relatively large translations, 
and therefore we expect a robust classifier to label all images in the sequence similarly.  Quantitatively, this is measured by the number of label flips in the sequence. 
We report the average number of label flips over all sequences of test images in  Table~\ref{tab:translations}. 
As the table shows, Wasserstein gradient regularization improves the robustness of the classifier to translations.

\begin{figure}[H]
    \centering
    \includegraphics[width = 10cm, clip=true, trim=.4cm .8cm 0cm .4cm]{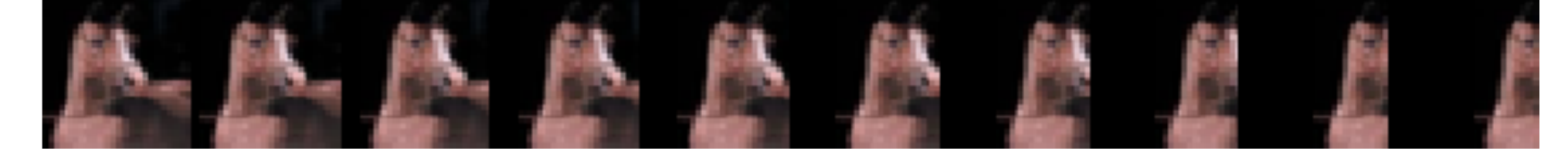}
    \caption{Robust classifiers should be invariant to natural variations of the data. Shown are horizontal translations of an image from CIFAR-10.}
    \label{fig:horse}
\end{figure}

\begin{table}[H]
\title{Label flips on natural in-class perturbations}
\centering
    \begin{tabular}{| l | c c c |}
                 \hline
         Perturbation \textbackslash{} Regularizer & None  &  Euclidean grad. & Wasserstein grad. \\
         \hline 
         Horizontal translation  & $10.009$ & $7.898$ &  $\mathbf{6.488}$   \\
         \hline 
         Vertical translation &  $\phantom{1}9.920$ & $9.437$ & $\mathbf{7.956}$  \\
         \hline
    \end{tabular} 
    \smallskip 
    
    

    \caption{Average number of prediction flips on sequences of translated test images from CIFAR-10. 
    The classifiers were trained on the clean CIFAR-10 training set with no data augmentation, with either no regularization, Euclidean Tikhonov regularization, or Wasserstein Tikhonov regularization. 
    }
    \label{tab:translations}
\end{table}

\section{Related works}
\label{sec:related}

There are many works related to Wasserstein geometry, robustness, regularization. In this section we briefly mention some of the literature in relation to our discussion in this article. 

{\em Adversarial robustness.} 
Previous works have investigated postprocessing with Jacobian regularization~\citep{DBLP:journals/corr/abs-1803-08680} and cross Lipschitz regularization~\citep{NIPS2017_6821}, whereby the input space was modeled as Euclidean space. 
The duality of attack norms and Lipschitz norms has been discussed as well~\citep{finlay2019improved}. 
%
Perturbation based regularization has been proposed~\citep{NIPS2018_7324}, which penalizes the negative size of the deep fool attack in proportion to the size of the input. 
Gaussian data augmentation was proposed too~\citep{Zantedeschi:2017:EDA:3128572.3140449}, but evaluated by Monte Carlo samples and using Euclidean space. 
Recently, the tradeoff between natural 
and robust classification errors was studied, 
leading to a training objective with an added term of the form $\mathbb{E}_x[\max_{x'\in B(x,\epsilon)}\phi(f(x)f(x')/\lambda)]$ \citep{DBLP:journals/corr/abs-1901-08573}. 
Similar to ours, this approach penalizes the variability of the classifier, but it is not incorporating priors about the geometry of the classes. 
While working on this article, we became aware of a work using modified Sinkhorn iterations to approximate projections of adversarial examples onto a Wasserstein ball~\citep{pmlr-v97-wong19a}. This is similar to the adversarial norm constraint that we suggested here. However, our approach is based on a Riemannian metric formulation, which allows us to obtain a very simple quadratic form approximation of the norm and also to integrate a generative  noise model (Wasserstein diffusion) into an effective smoothness regularizer. 

{\em Wasserstein sample space.} 
Optimal transport has been applied to the design of training objectives~\citep{WGAN}. Recently this has been combined with a Wasserstein metric on image data space~\citep{82084}. The Kantorovich duality of the Wasserstein training objective leads to a Lipschitz constraint on the discriminator networks, which itself is computed in Wasserstein space. This approach is known as the Wasserstein of Wasserstein GAN (WWGAN). 
Our regularization penalty derived from expanding a Wasserstein Gaussian noise variable in the input space of a classifier also includes a term that penalizes the Wasserstein gradient norm. However, our analysis has a different motivation and interpretation and also reveals higher order expansion terms. 

{\em Wasserstein diffusion.}
The noise in Wasserstein space has been studied in continuous~\citep{vonrenesse2009} and discrete~\citep{Li2018_geometry} state spaces. 
The present work seems to be the first to apply Wasserstein type noise in machine learning. 
In this trend, the definition and efficient computation of the Riemannian volume form on Wasserstein space remains an open problem for future work. 

{\em Wasserstein Information Geometry.} The Wasserstein metric is gaining traction not only in the design of training objectives and in the definition of geometric structures on the sample space of generative models, but also in the development of natural gradients and optimizers \citep{Li2018_geometry, Li2018,LiMontufar2018_ricci}. Recently this has been applied to GANs \citep{LinLiOsherMontufar2018_wasserstein}. In this paper, we derive and apply second order calculus of Wasserstein geometry for improving generalization, especially improving the robustness to adversarial attacks. 

{\em Robustness and regularization.} 
Wasserstein balls have appeared in the context of robust density estimation  \citep{shafieezadeh2017regularization}, where they are also related to a form of Tikhonov regularization. 
Wasserstein distributionally robust stochastic optimization has been related to regularization by certain empirical gradient norms \citep{gao2017wasserstein}. 
Close to our derivations, albeit not involving Wasserstein geometry, is the work by \cite{doi:10.1162/neco.1995.7.1.108}, which shows that training with noise is equivalent to {T}ikhonov regularization.

\section{Discussion}
\label{sec:discussion}
Training with input noise or data augmentation in general is known as an effective form of regularization to obtain classifiers that are more robust to natural variations of the data, or to reduce the sensitivity to perturbations. 
These methods usually have a high cost in terms of the number of examples needed and the cost of computing each of them (especially in the case of adversarial data augmentation obtained by iterated gradient methods). 
Another problem is that usually noise models and adversarial examples need to be restricted to tiny norm values to ensure that they remain within the class of the perturbed example. 
Smoothness regularizers based on $L^p$ metrics are usually limited in the same way. 
In this paper we follow the idea that the space of inputs is not Euclidean and that smoothness priors should be implemented with respect to an appropriate metric, which in turn would allow us to apply higher levels of regularization without hurting test performance. 
We propose to use the Wasserstein-2 metric to capture semantically meaningful neighborhoods of images. 
As we show, the Wasserstein diffusion smoothness regularizer arises naturally by expanding and integrating the loss with respect to Wasserstein Gaussian noise on the inputs. 
We obtain an effective penalty that can be computed very efficiently, saving computation compared with adversarial data augmentation, and has a negligible overhead over $L^2$ gradient penalties. 
Preliminary experimental results indicate that our methods can significantly improve robust generalization performance on CIFAR-10. 
We obtain models that are robust not only to small perturbations (the usual setting in adversarial robustness literature), but also to large scale in-class perturbations, such as translations. 
We think that this is conceptually an important step towards learning models that generalize better in relation to all types of natural variations of the input data, not only small perturbations.

\subsubsection*{Acknowledgments}
This research has received funding from AFOSR MURI FA9550-18-1-0502. 
YD has received funding from the National Science Foundation Graduate Research Fellowship under Grant No.~DGE-1650604. 
Part of this research was performed at the Institute for Pure and Applied Mathematics (IPAM), which is supported by the National Science Foundation (Grant No.~DMS-1440415). 
This project has received funding from the European Research Council (ERC) under the European Union's Horizon 2020 research and innovation programme (grant agreement n\textsuperscript{o} 757983). 

\bibliography{wgm}
\bibliographystyle{abbrvnat}

\appendix

\section{Proof of Theorem~\ref{thm}}\label{secA}
A Riemannian metric $g$ defines an inner product between tangent vectors of the input space at each possible location. 
We choose standard coordinates for the input space $\mathcal{X}=\mathbb{R}^n$ and write $g(\xi, \zeta) = (\xi,\zeta)_g = \xi^\top G(x) \zeta$ for any pair $\xi,\zeta\in T_x\mathcal{X}$. We implicitly identify $T_x\mathcal{X}$ and $\mathcal{X}$ so that adding a tangent vector $\xi\in T_x\mathcal{X}$ to an input vector $x\in\mathcal{X}$ makes sense. 
The Riemannian gradient with the metric $g$ is given by $\nabla_g f(x) = G^{-1}(x)\nabla f(x)$, where $\nabla$ is the ordinary gradient. 
This is also known as the natural gradient. 

\begin{proof}[Proof of Theorem~\ref{thm}]
We expand the error function $l$ around a data point $x$ with added noise $\xi$ in the Riemannian space $(\mathcal{X},g)$. 
We obtain 
\begin{multline*}
l(f(x+\xi), y) = 
l(f(x),y) + l'(f(x),y) (\nabla_g f(x),\xi)_g \\
+ \frac12 l''(f(x),y) (\nabla_g f(x), \xi)_g^2 + \frac12 l'(f(x),y) \sum_{i,j} \xi_i\xi_j (\nabla_g^2 f(x))_{ij} + o(\|\xi\|_g^2). 
\end{multline*}
We discuss the individual terms in turn. 
The zero order term is just the unperturbed loss. 
On taking the expectation value with respect to $\xi$ given $x$, the linear term vanishes when we assume that the perturbations have zero mean, $\mathbb{E}_{p(\xi|x)}[\xi]=0$. 
If the perturbation does not have zero mean, we obtain 
	\begin{equation*}
	\mathbb{E}_{p(\xi|x)}[(\nabla_g f(x),\xi)_g] = \mathbb{E}_{p(\xi|x)}[(G^{-1}(x)\nabla f(x))^\top G(x) \xi] = \nabla f(x)^\top \mathbb{E}_{p(\xi|x)}[\xi]. 
	\end{equation*}
	
	For the first quadratic term, when $\mathbb{E}_{p(\xi|x)}[\xi\xi^\top]= \eta^2 G^{-1}(x)$, we obtain 
	\begin{gather*}
	\mathbb{E}_{p(\xi|x)}[(\nabla_g f(x)^\top G(x) \xi )^2]  = 
	\eta^2 \nabla_g f(x) ^\top G(x) \nabla_g f(x)
	=
	\eta^2 \| \nabla_g f\|_g^2. 
	\end{gather*}
	For the second quadratic term, again when $\mathbb{E}_{p(\xi|x)}[\xi \xi^\top] = \eta^2 G^{-1}(x)$, we obtain 
	\begin{equation*}
	\mathbb{E}_{p(\xi|x)} [\xi^\top \operatorname{Hess} f(x) \xi]  = \eta^2 \Delta_g f(x).
	\end{equation*}
	Here the Laplace-Beltrami operator is $$\Delta_g f =\sum_{j,k} g^{jk} \frac{\partial^2 f}{\partial x^j \partial x^k} - g^{jk} \Gamma^l_{jk}\frac{\partial f}{\partial x^l},$$  where $\Gamma^l_{jk}$ is the Christoffel symbol. 
\end{proof}

\section{Riemannian calculus in Wasserstein space over discrete states}
\label{app:Riemannian_calculus}

In this section, we review the definition of Wasserstein-2 Riemannian metric on discrete states founded in \citep{chow2012,EM1,M} and developed in  \citep{ChowLiZhou2018_entropy} and \citep{Li2018}. See Wasserstein Riemannian calculus in \citep{Li2018_geometry}. 
We also discuss practical implementations of the Laplace-Beltrami operator appearing in the Wasserstein Tikhonov regularizer. 

Consider $I=\{1,\ldots, n\}$. The probability simplex on $I$ is the set 
\begin{equation*}
\mathcal{P}_+(I) = \Big\{(x_1,\ldots, x_n)\in \mathbb{R}^n \colon \sum_{i=i}^n x_i=1,\quad  x_i>0 0\Big\}. 
\end{equation*}
Here $x=(x_1,\ldots, x_n)$ is a probability vector with coordinates $x_i$ 
corresponding to the probabilities assigned to each node $i\in I$. 

We next define the Wasserstein-2 metric tensor on $\mathcal{P}_+(I)$. 
This is given in terms of a undirected graph with weighted edges, $\mathcal{G}=(I, E, \omega)$, where $I$ is the vertex set, $E\subseteq {I\choose 2}$ 
is the edge set, and $\omega=(\omega_{ij})_{i,j\in I} \in \mathbb{R}^{n\times n}$ is a matrix of edge weights satisfying 
$$
\omega_{ij}=
\begin{cases}
\omega_{ji}>0, & \text{if $(i,j)\in E$}\\
0, & \text{otherwise}
\end{cases}.
$$ 
The set of neighbors (adjacent vertices) of $i$ is denoted by $N(i)=\{j\in V\colon (i,j)\in E\}$. 
The normalized volume form on node $i\in I$ is given by $d_i=\frac{\sum_{j\in N(i)}\omega_{ij}}{\sum_{i=1}^n\sum_{i'\in N(i)}\omega_{ii'}}$. 

The graph structure $\mathcal{G}=(I, E, \omega)$ induces a graph Laplacian matrix function. 
\begin{definition}[Weighted Laplacian matrix]
	\label{def2}
	Given an undirected weighted graph $\mathcal{G}=(I,E,\omega)$, with $I=\{1,\ldots, n\}$, the matrix function $L(\cdot):\mathbb{R}^n\rightarrow \mathbb{R}^{n\times n}$ is defined as 
	\begin{equation*}
	L(p)=D^{\ts}\Lambda(x)D,\quad x=(x_i)_{i=1}^n\in \mathbb{R}^n,
	\end{equation*}
	where 
	\begin{itemize}
		\item 
		$D \in \mathbb{R}^{|E|\times n}$ is the discrete gradient operator given by 
		\begin{equation*} 
		D_{(i,j)\in {E}, k\in V}=\begin{cases}
		\sqrt{\omega_{ij}}, & \text{if $i=k$, $i>j$}\\ 
		-\sqrt{\omega_{ij}}, & \text{if $j=k$, $i>j$}\\
		0, & \text{otherwise}, 
		\end{cases}
		\end{equation*}
		\item $-D^{\ts}\in \mathbb{R}^{n\times |E|}$ is the oriented incidence  matrix,
		and 
		\item $\Lambda(x)\in \mathbb{R}^{|E|\times |E|}$ is a weight matrix depending on $x$, 
		\begin{equation*}
		\begin{split}
\Lambda(x)_{(i,j)\in E, (k,l)\in E} 
		=\begin{cases}
		\frac{1}{2}(\frac{1}{d_i}x_i+\frac{1}{d_j}x_j), & \text{if $(i,j)=(k,l)\in E$}\\ 
		0, & \text{otherwise}.
		\end{cases}
	    \end{split}
		\end{equation*}
	\end{itemize}
\end{definition}

We are now ready to present the Wasserstein-2 metric tensor. Consider the tangent space of $\mathcal{P}_+(I)$ at~$x$, 
\begin{equation*}
T_x\mathcal{P}_+(I) = \Big\{(\sigma_i)_{i=1}^n\in \mathbb{R}^n\colon  \sum_{i=1}^n\sigma_i=0 \Big\}.
\end{equation*}

\begin{definition}[Wasserstein-2 metric tensor]\label{d9}
	The inner product 
	$g \colon T_x\mathcal{P}_+(I)\times T_x\mathcal{P}_+(I)  \rightarrow \mathbb{R}$ takes any two tangent vectors $\sigma_1, \sigma_2\in T_p\mathcal{P}_+(I)$ to 
	\begin{equation*}
	g_x(\sigma_1,\sigma_2):={\sigma_1}^{\ts}L(x)^{-1}\sigma_2,\quad \text{for any $\sigma_1,\sigma_2\in T_x\mathcal{P}_+(I)$}, 
	\end{equation*}
	where $L(x)^{\dd}$ is the pseudo inverse of $L(x)$. 
\end{definition}
Here $L(x)^{\dd}_{ij}$ plays the role of $g_{ij}$ in the last subsection \ref{secA}. Using this metric tensor, the Riemannian calculus in $(\mathcal{P}_+(I), g)$ has the following formulations.

 (i) The Christoffel symbol $\Gamma^W_x\colon T_x\mathcal{P}_+(I)\times T_x\mathcal{P}_+(I)\rightarrow T_x\mathcal{P}_+(I)$ forms: 
\begin{equation*}
\Gamma^W_x(\sigma_1,\sigma_2)=-\frac{1}{2}\big[L(\sigma_1)L(\rho)^{-1}\sigma_2+L(\sigma_2)L(\rho)^{-1}\sigma_1\big]+\frac{1}{2}L(\rho)\Big(\nabla_G L(\rho)^{-1}\sigma_1	\circ\nabla_G L(\rho)^{-1}\sigma_2\Big)\ ,
\end{equation*}
where $\sigma_1$, $\sigma_2\in T_x\mathcal{P}_+(I)$ and 
$$\big(\nabla_G L(\rho)^{-1}\sigma_1	\circ \nabla_G L(\rho)^{-1}\sigma_2\big)=\Big(\frac{1}{2d_i}\sum_{j\in N(i)}(\nabla_{ij}L(\rho)^{-1}\sigma_1) (\nabla_{ij}L(\rho)^{-1}\sigma_2) \Big)_{i=1}^n\in \mathbb{R}^n.$$ 

(ii) Given $F\in C^{\infty}(\mathcal{P}_+(I))$, the Riemannian gradient is 
 \begin{equation*}
\textrm{grad}_W F(x)=\Big(L(x)^{\dd}\Big)^{\dd}\nabla F(x)=L(x)\nabla F(x),
\end{equation*}
where $\nabla$ is the Euclidean $L^2$ derivative w.r.t. $x$. 

(iii) The Riemannian Hessian operator $\textrm{Hess}_W F(\rho)\colon T_{\rho}\mathcal{P}_+(M)\times T_{\rho}\mathcal{P}_+(M)\rightarrow \mathbb{R}$ is given by 
\begin{equation*}
\begin{split}
\textrm{Hess}_W F(x)(\sigma_1, \sigma_2)=&\sigma_1^{\ts}\nabla^2F(x)\sigma_2+\sum_{i=1}^n\Gamma_x^W(\sigma_1,\sigma_2)_i \nabla_{x_i}F(x),
\end{split} 
\end{equation*}
where $\sigma_1$, $\sigma_2\in T_x\mathcal{P}_+(I)$.

(iv) The Riemannian volume is given by 
\begin{equation*}
\textrm{vol}_{W}(x)=\textrm{det}(L(x))^{-\frac{1}{2}}=\Pi_{i=1}^{n-1}\lambda_i(x)^{-\frac{1}{2}}\ ,
\end{equation*}
where $\lambda_i(x)$ are the positive eigenvalues of $L(x)$. 

(v) The Laplacian--Beltrami operator is given by   
\begin{equation*}
\begin{split}
\Delta_W F(x)=\textrm{tr}(L(x)\Big(\frac{\partial^2f}{\partial x_i\partial x_j}\Big)_{1\leq i,j\leq n})+\nabla f(x)^{\ts}L(x)\nabla \log\textrm{det}(L(x))^{-\frac{1}{2}}.
\end{split}
\end{equation*}
If the reference measure is a Lebesgue measure in a simplex, the modified Wasserstein Laplacian operator satisfies
\begin{equation}
\begin{split}
\tilde{\Delta}_Wf(x)=&\textrm{tr}(L(x)(\frac{\partial^2f}{\partial x_i\partial x_j})_{1\leq i,j\leq n})   \\
=& \sum_{(i,j)\in E} \omega_{ij} \left(\frac{\partial^2}{\partial x_i^2} f(x) -2\frac{\partial^2}{\partial x_i\partial x_j}f(x)+\frac{\partial^2}{\partial x_j^2} f(x)\right)\frac{x_i/d_i+x_j/d_j}{2}.
\label{eq:approx_laplace}
\end{split}
\end{equation}

\section{Efficient implementation of the Wasserstein gradient norm}
\label{sec:computation_gradient}

To compute \eqref{eq:grad_norm} in practice, we define a suitable similarity graph $\mathcal{G}=(V,E,\omega)$ for the space of images, displaying translation invariance and symmetries. 
First, there is the invariance with respect to pixel translations.  
Symmetries arise since the distance from a pixel to two spatially opposite pixels is equal. 
In addition, each term in the sum in \eqref{eq:grad_norm} can be decomposed into the product of linear relations between values in nodes $i$ and $j$. Each linear relation (e.g $\nabla_{X_{i}} f- \nabla_{X_j} f $) is computed via a convolution, in this case on $\nabla_{X} F$.  Convolutions may replace a linear product owing it to the described symmetries and invariances. For each relative neighbor direction we define a convolutional filter with the number of filters equal to the number of neighbors. 
For the truncated similarity graph for the Wasserstein distance, the edge set $E$ is sparse and the number of convolutional filters is reduced considerably from $n^2$. 
We therefore can calculate all pairs $\nabla_{X_i} f - \nabla_{X_j} f$ with a given relative neighbor relation by passing $X$ and $\nabla_{X} f$ through of kernels $K_{\mathcal{O}_1} \dots K_{\mathcal{O}_d}$. In this case a neighbor relation, is defined as the geometrical pattern between two pixels. A pixel located in position $(10,10)$ satisfies a neighbor relation $(1,2)$ with a pixel in location $(11,12)$. The neighbor relation is indeed invariant to the global position of the pixel which allows for the use of convolutions. 
Given a ground metric, we enumerate all non-zero neighbor relations as $\mathcal{O}_1, \dots \mathcal{O}_d$, for truncated distances this is much smaller than the complete edge graph. 
For each neighbor relation $\mathcal{O}_k$ we associate a zero-valued kernel that equals $1$ and $-1$ in the corresponding $\mathcal{O}_k$ pixels, we denote the gradient kernels as $K_{\mathcal{O}_k}$. Likewise, we apply the same $\mathcal{O}_k$ pattern, now with $\frac{1}{2}$, $\frac{1}{2}$ in the corresponding neighbor pattern pixels to obtain the terms $\frac{X_i/d_i+X_j/d_j}{2}$. For each $i,j$ we denote the input kernels as $M_{\mathcal{O}_k}$. Applying entry-wise multiplication ($\odot$) and a summation collapsing all pixel locations and channels then yields an efficient and general method of calculating the Wasserstein gradient $\|\operatorname{grad} f\|_{W_{2, d(\Omega)}}$ for general local cost metrics on highly optimized convolution. 

\begin{algorithm}
\caption{Wasserstein gradient norm $\|\operatorname{grad}f(X) \|_{W}^2$.}
\label{alg:buildL}
\begin{algorithmic}[1]
\REQUIRE{The pixel graph $\mathcal{G} = (V,E,\omega)$; local weights $(w_{ij})$; neighbor relation tuples arranged symmetrically $\mathcal{O}_{1} \dots \mathcal{O}_{d}$}
\REQUIRE{Euclidean gradient $\nabla_{X} f$}
\STATE{$\textit{Wasserstein-grad} \gets 0$}
\FOR{neighbor relations $k =1,\ldots,d$}
    \STATE{Build kernel $K_{\mathcal{O}_k}$ to compute $\nabla _{X_i}f - \nabla _{X_{\mathcal{O}_k (i)}}f$}
    \STATE{Build corresponding kernel $M_{\mathcal{O}_k}$  to compute $\frac{X_i}{2d_i} + \frac{X_{\mathcal{O}_k}}{2d_{\mathcal{O}_k}}$}
    \STATE{$H \gets K_{\mathcal{O}_k}(\nabla_{X}f)$ } 
    \STATE{$V \gets M_{\mathcal{O}_k}(X)$ } 
    \STATE{$H \gets H \odot H $\quad (entry-wise multiplication)}
    \STATE{$W \gets H  \odot V$}
    \STATE{$\textit{Wasserstein-grad} \gets \textit{Wasserstein-grad} + \text{sum}(W) $}
\ENDFOR
\STATE{\textbf{Return }$\|\operatorname{grad}f(X) \|_{W}^2= \textit{Wasserstein-grad}$}
\end{algorithmic}
\end{algorithm}

\section{Details on the experimental setup}
\label{sec:details_exp1}

For our experiments, we use the CIFAR-10 dataset, and perform white-box attacks on the ResNet20 network. For training, we fixed the batch size of $128$, and used SGD with momentum and weight decay, where the momentum value is $0.9$ and the weight-decay value is $10^{-4}$. We start with a learning rate of $0.1$, and at epoch $100$ and $150$ we divide the learning rate by $10$ each time.

We examine the case of training the ResNet-20 network on the CIFAR-10 dataset, where the only data augmentation performed is normalization. 
This achieves a test accuracy of 
$83.71\%$. 
We then examine the effect of modifying the loss objective with either the Euclidean or Wasserstein gradient penalties of the original loss, namely we use the loss
    \begin{equation*}
        \ell(f(x), y) + \eta^2(\nabla_x \ell(f(x), y), G(x)^{-1}\nabla_x \ell(f(x), y) ) , 
    \end{equation*}
where $\nabla_x$ is the Euclidean gradient and $G(x)\in\mathbb{R}^{d\times d}$ represents the metric used in sample space. For the Wasserstein gradient norm, $G(x)^{-1}=L(x)$. For the Euclidean gradient norm, $G(x)=I$.

\end{document}